\documentclass[letterpaper, 10 pt, conference]{ieeeconf}

     \IEEEoverridecommandlockouts 
    \makeatletter

    \let\proof\@undefined
    \let\endproof\@undefined
    \makeatother
    \let\labelindent\relax

    \usepackage{jhoagg}
    \usepackage{mathtools}
    \usepackage{amsfonts}
    \usepackage{amssymb,latexsym}
    \usepackage[psamsfonts]{eucal}
    \usepackage{amsthm}
    \usepackage{graphicx}
    \usepackage[inline]{enumitem}
    \usepackage{indentfirst}
    \usepackage{setspace}
    \usepackage{microtype}
    \usepackage{threeparttable}
    \usepackage{float}
    \usepackage{cleveref}
    \usepackage{afterpage}
    \usepackage{placeins}
    \usepackage{cancel}
    \usepackage{cite}
    \usepackage{leftidx}
    \usepackage{breqn}
    \usepackage[export]{adjustbox}
    \usepackage{comment}
    \usepackage[ruled, linesnumbered, lined]{algorithm2e}
    \usepackage[dvipsnames]{xcolor}
    \usepackage{xcolor}
    \usepackage{soul}
    \usepackage{siunitx}

    \usepackage{tikz}
    \usepackage{pgf}
    \usepackage{pgfplots}
    \usepgflibrary{plothandlers}
    \usepgfplotslibrary{external} 
    \pgfkeys{/pgf/number format/.cd,1000 sep={}}  
    \usetikzlibrary{calc,shapes.misc,plotmarks}
    \pgfplotsset{compat=1.13}
    
    \let\originalleft\left
    \let\originalright\right
    \renewcommand{\left}{\mathopen{}\mathclose\bgroup\originalleft}
    \renewcommand{\right}{\aftergroup\egroup\originalright}

    \newcounter{thm} 
    \newtheorem{theorem}[thm]{\indent Theorem}
    
    \newtheorem{assumption}{\indent Assumption}
    
    \newtheorem{proposition}{\indent Proposition}
    
    \newtheorem{lemma}{\indent Lemma}
    
    \newtheorem{corollary}{\indent Corollary}
    
    \newtheorem{definition}{\indent Definition}

    \newtheorem{example}{\indent Example}

    \newtheorem{Simulation}{Simulation}

    \newtheorem{fact}{\indent Fact}
    
    \newtheorem{conjecture}{\indent Conjecture}
    
    \newtheorem{experiment}{\indent Experiment}
    
    \allowdisplaybreaks

    \renewcommand{\theenumi}{{\it (\alph{enumi})}}
    \renewcommand{\labelenumi}{\theenumi}

    \usepackage{cite}
    \usepackage{accents}
    
    \newlength\figureheight 
    \newlength\figurewidth

    \allowdisplaybreaks
    
    \graphicspath{ {Figures/} }
    \DeclareGraphicsExtensions{.png}
    
    \DeclareMathAlphabet{\mathcal}{OMS}{cmsy}{m}{n} 

    \crefname{equation}{}{}
    
    \usepackage[ruled, linesnumbered, lined]{algorithm2e}
    \SetKwInput{KwInit}{Initialize}
    \SetKwFunction{SafeSet}{SafeSet}
    \SetKwFunction{GetSafe}{GetSafe}
    \SetKwFunction{GetUnsafe}{GetUnsafe}
    
    \SetCommentSty{mycommfont}

    \newcommand{\bd}[0]{\mbox{bd }}

\begin{document}

\title{Safe Navigation in Unmapped Environments for Robotic Systems with Input Constraints}

\author{Amirsaeid Safari and Jesse B. Hoagg
\thanks{A. Safari and J. B. Hoagg are with the Department of Mechanical and Aerospace Engineering, University of Kentucky, Lexington, KY, USA. (e-mail: amirsaeid.safari@uky.edu, jesse.hoagg@uky.edu).}
\thanks{This work is supported in part by the National Science Foundation (1849213) and Air Force Office of Scientific Research (FA9550-20-1-0028).}
}
\maketitle

\begin{abstract}
This paper presents an approach for navigation and control in unmapped environments under input and state constraints using a composite control barrier function (CBF). 
We consider the scenario where real-time perception feedback (e.g., LiDAR) is used online to construct a local CBF that models local state constraints (e.g., local safety constraints such as obstacles) in the \textit{a priori} unmapped environment.  
The approach employs a soft-maximum function to synthesize a single time-varying CBF from the $N$ most recently obtained local CBFs.
Next, the input constraints are transformed into controller-state constraints through the use of control dynamics.
Then, we use a soft-minimum function to compose the input constraints with the time-varying CBF that models the \textit{a priori} unmapped environment.
This composition yields a single relaxed CBF, which is used in a constrained optimization to obtain an optimal control that satisfies the state and input constraints. 
The approach is validated through simulations of a nonholonomic ground robot that is equipped with LiDAR and navigates an unmapped environment. 
The robot successfully navigates the environment while avoiding the \textit{a priori} unmapped obstacles and satisfying both speed and input constraints. 
\end{abstract}




\section{Introduction}
Safe autonomous navigation in unknown or changing environments is a critical challenge in robotics with applications ranging from search and rescue~\cite{hudson2021heterogeneous} to environmental monitoring~\cite{kress2009temporal} and transportation~\cite{schwarting2018planning}. 
Control barrier functions (CBFs) are a tool for ensuring state constraints (e.g., safety) in robotic systems by providing a set-theoretic method to obtain forward invariance of a specified safe set~(e.g., \cite{wieland2007constructive,ames2019control,ames2016control}). 
However, effective application of CBFs in real-world scenarios faces several challenges, including: (1) online construction of CBF for unknown environments, and (2) implementation of CBFs under input constraints (e.g. actuator limits).

Regarding the first challenge, barrier function approaches often assume that the barrier functions are constructed offline using \textit{a priori} knowledge of the environment~\cite{ames2019control}. 
Online methods have been used to synthesize CBFs from sensor data, including support-vector-machines approaches~\cite{srinivasan2020synthesis} and Gaussian-process approaches~\cite{khan2022gaussian}.
However, when new sensor data is obtained, the barrier function model must be updated, which often results in discontinuities that can be problematic for ensuring forward invariance and thus, state constraint satisfaction.

Nonsmooth barrier functions \cite{glotfelter2017nonsmooth} and hybrid nonsmooth barrier functions \cite{glotfelter2019hybrid} can be used to partially address the challenges that arise from updating barrier functions in real time. 
However, \cite{glotfelter2017nonsmooth, glotfelter2019hybrid} are not applicable for relative degree greater than one. 
Thus, these approaches cannot be directly applied to ground robots with nonnegligible inertia or unmanned aerial vehicles with position constraints (e.g., obstacles).

For systems with arbitrary relative degrees, \cite{safari2024TSCT} uses a smooth time-varying construction to address online updating of barrier functions. 
In particular, \cite{safari2024TSCT} considers scenarios where real-time local perception data can be used to construct a local barrier function that models local state constraints. 
Then, \cite{safari2024TSCT} uses a smooth time-varying soft-maximum construction to compose the $N$ most recently obtained local barrier functions into a single barrier function whose zero-superlevel set approximates the union of the $N$ most recently obtained local subsets. 
This smooth time-varying barrier function is used to construct controls that guarantee satisfaction of state constraints based on online perception information for systems with arbitrary relative degree. 
However, \cite{safari2024TSCT} does not address input constraints.


Input constraints (e.g., actuator limits) often present a challenge in constructing a valid CBF. 
Specifically, verifying a candidate CBF under input constraints can be challenging and this complexity is further exacerbated in multi-objective safety scenarios, where multiple safety constraints must be satisfied simultaneously. 
Offline approaches using sum-of-squares optimization \cite{wang2018} or state space gridding have been proposed\cite{tan2022compatibility}, but these methods may lack real-time applicability. 
An alternative online approach to obtain forward invariance subject to input constraints is to use a prediction of the system trajectory under a backup control \cite{backupautomatic,gurriet2020,chen2020}. 
However, \cite{backupautomatic,gurriet2020,chen2020} all rely on a prediction of the system trajectories into the future.


A different approach to addressing both state and input constraints is presented in \cite{compositionACC}, which uses smooth soft-minimum and soft-maximum functions to compose multiple barrier functions into a single barrier function.
Smooth compositions are also considered in \cite{lindemann2018control}. 
In \cite{compositionACC}, control dynamics are introduced where the control signal is expressed as an algebraic function of the controller state, allowing input constraints to be treated as additional barrier functions in the state of the controller. 
Then, the smooth composition is used to compose all input and state constraints into a single barrier function. 
However, \cite{compositionACC} does not consider the challenges of \textit{a priori} unmapped environments.

This article presents an approach that simultaneously addresses input constraints and online barrier function construction in \textit{a priori} unmapped environments. 
The main contribution of this article is a control method that allows for safe navigation in unmapped environments while satisfying both state and input constraints.
The approach leverages the time-varying soft-maximum barrier function from \cite{safari2024ACC, safari2024TSCT} to compose the $N$ most recently obtained local barrier function, which are constructed from local online perception data.
Next, we adopt ideas from \cite{compositionACC} to transform the input constraints into controller-state constraints through the use of control dynamics.
Then, we use a soft-minimum function to compose the input constraints with the time-varying CBF that models the \textit{a priori} unmapped environment.
This composition yields a single relaxed CBF that is used in a constrained optimization, which is solved in closed form to obtain a feedback control that guarantees state and input constraint satisfaction. 
This closed-form optimal feedback control ensures safety in an \textit{a priori} unknown environment (e.g., avoiding obstacles), while satisfying other known state constraints (e.g., speed limits) as well as input constraints (e.g., actuator limits). 
The approach is validated through simulations of a nonholonomic ground robot that is equipped with LiDAR and navigates an unmapped environment. 
The robot successfully navigates the environment while avoiding the \textit{a priori} unmapped obstacles and satisfying both speed and input constraints.



\section{Notation}

The interior, boundary, and closure of the set $\mathcal{A} \subseteq \mathbb{R}^n$ are denoted by $\mbox{int}~\mathcal{A}$, $\mbox{bd}~\mathcal{A}$, and $\mbox{cl}~\mathcal{A}$ respectively.
Let $\mathbb{N} \triangleq \{ 0, 1, 2, \ldots \}$, and let $\| \cdot \|$ denote the $2$ norm on $\BBR^n$.

Let $\zeta:[0,\infty) \times \BBR^n \to \BBR$ be continuously differentiable. 
Then, the partial Lie derivative of $\zeta$ with respect to $x$ along the vector fields of $\nu:\mathbb{R}^n \to \mathbb{R}^{n \times \ell}$ is defined as 
\begin{equation*}
L_\nu \zeta(t,x) \triangleq \frac{\partial \zeta(t,x)}{\partial x} \nu(x).
\end{equation*}
In this paper, we assume that all functions are sufficiently smooth such that all derivatives that we write exist and are continuous.

A continuous function $a \colon \BBR \to \BBR$ is an \textit{extended class-$\SK$ function} if it is strictly increasing and $a(0)=0$.

Let $\kappa>0$, and consider $\mbox{softmin}_\kappa \colon \mathbb{R}^N \to \mathbb{R}$ and $\mbox{softmax}_\kappa \colon \mathbb{R}^N \to \mathbb{R}$ defined by
\begin{gather}
\mbox{softmin}_\kappa (z_1,\cdots,z_N) \triangleq -\frac{1}{\kappa}\log\sum_{i=1}^Ne^{-\kappa z_i},\label{eq:softmin}\\
\mbox{softmax}_\kappa (z_1,\cdots,z_N) \triangleq \frac{1}{\kappa}\log\sum_{i=1}^Ne^{\kappa z_i} - \frac{\log N }{\kappa},\label{eq:softmax}
\end{gather}
which are the log-sum-exponential \textit{soft minimum} and \textit{soft maximum}, respectively. 
The next result relates the soft minimum to the minimum and the soft maximum to the maximum. 
See \cite{safari2024TSCT} for a proof.

\begin{proposition}
\label{fact:softmin_limit}
\rm
Let $z_1,\cdots, z_N \in \mathbb{R}$. 
Then,
\begin{align}
  \min \, \{z_1,\cdots,z_N\} - \frac{\log N }{\kappa} 
    &\le \mbox{softmin}_\kappa(z_1,\cdots,z_N) \nn \\
    &\le \min \, \{z_1,\cdots,z_N\}\label{eq:softmin_inequality},
\end{align}
and
\begin{align}
    \max\,\{z_1,\cdots,z_N\}  - \frac{\log N}{\kappa}
 &\le \mbox{softmax}_\kappa(z_1,\cdots,z_N) \nn \\
 &\le \max \, \{z_1,\cdots,z_N\}\label{eq:softmax_inequality}.
\end{align}
\end{proposition}

\Cref{fact:softmin_limit} shows that $\mbox{softmin}_\kappa$ and $\mbox{softmax}_\kappa$ lower bound minimum and maximum, respectively.
\Cref{fact:softmin_limit} also shows that $\mbox{softmin}_\kappa$ and $\mbox{softmax}_\kappa$ approximate the minimum and maximum in the sense that they converge to minimum and maximum, respectively, as $\kappa \to \infty$.

The next result is a consequence of \Cref{fact:softmin_limit}.
The result shows that soft minimum and soft maximum can be used to approximate the intersection and the union of zero-superlevel sets, respectively.
See \cite{safari2024TSCT} for a proof.

\begin{proposition}
\label{prop:softmin_softmax_sets}
\rm
For $i \in \{ 1,2,\ldots,N\}$, let $\zeta_i \colon \BBR^n \to \BBR$ be continuous, and define 
\begin{gather}
\SD_i \triangleq \{ x \in \BBR^n \colon \zeta_i(x) \ge 0 \}\label{prop2.1},\\
\SX_{\kappa} \triangleq \{ x \in \BBR^n \colon \mbox{softmin}_\kappa(\zeta_1(x),\cdots,\zeta_N(x)) \ge 0 \}\label{prop2.2},\\
\SY_{\kappa} \triangleq \{ x \in \BBR^n \colon \mbox{softmax}_\kappa(\zeta_1(x),\cdots,\zeta_N(x)) \ge 0 \}\label{prop2.3}.
\end{gather}
Then, 
\begin{equation*}
    \SX_{\kappa} \subseteq \bigcap_{i=1}^N \SD_i, \qquad \SY_{\kappa} \subseteq \bigcup_{i=1}^N \SD_i.
\end{equation*}
Furthermore, as $\kappa \to \infty$, $\SX_{\kappa} \to \bigcap_{i=1}^N \SD_i$ and $\SY_{\kappa} \to \bigcup_{i=1}^N \SD_i$.
\end{proposition}

\section{Problem Formulation}\label{sec:problem formulation}

Consider
\begin{equation}\label{eq:affine control}
\dot x(t) = f(x(t))+g(x(t)) u(t), 
\end{equation}
where $x(t) \in \mathbb{R}^n $ is the state, $x(0) = x_0 \in \mathbb{R}^n$ is the initial condition, $f: \mathbb{R}^n \to \mathbb{R}^n$ and $g: \mathbb{R}^n \to \mathbb{R}^{n \times m}$ are locally Lipschitz continuous on $\BBR^n$, $u:[0, \infty) \to \SU$ is the control, and the set of admissible controls is 
\begin{equation}\label{eq:U_def}
    \SU \triangleq \{ u\in \BBR^m: \phi_1(u) \ge 0, \ldots,  \phi_{\ell}( u) \ge 0\} \subseteq \BBR^m, 
\end{equation}
where $\phi_1,\ldots,\phi_{\ell}: \BBR^m \to \BBR$ are continuously differentiable. 
We assume that for all $u \in \SU$, $\phi^\prime_1(u) \neq 0,\ldots,\phi^\prime_\ell(u) \neq 0$.




Next, let  $\xi_1, \ldots, \xi_{p} : \BBR^n \to \BBR$ be continuously differentiable, and define the set of \textit{known state constraints}
\begin{equation} \label{eq:S_varphi}
    \SSS_\xi \triangleq \{ x \in \BBR^n \colon \xi_1(x) \geq 0, \ldots, \xi_{p}(x) \geq 0 \},
\end{equation}
which is the set of states that satisfy \textit{a priori} known constraints.
For example, $\SSS_\xi$ could be the the set of states that satisfy \textit{a priori} known limits on the velocity of an uncrewed air vehicle or a ground robot. 
We make the following assumption:
\begin{enumerate}[leftmargin=0.9cm]
	\renewcommand{\labelenumi}{(A\arabic{enumi})}
	\renewcommand{\theenumi}{(A\arabic{enumi})}

\item\label{con1_varphi}
There exists a positive integer $d$ such that for all $x \in \BBR^n$ and all $j \in \{ 1,2,\ldots,p\}$, $L_g \xi_j(x)=L_gL_f \xi_j(x)=\cdots=L_gL_f^{d-2} \xi_j(x)=0$; and for all $x \in \SSS_\xi$, $L_g L_f^{d-1} \xi_j(x) \neq 0$.
\end{enumerate}

Assumption \ref{con1_varphi} implies $\xi_{j}$ has relative degree $d$ on $\SSS_\xi$. 
For simplicity, we assume that the relative degree $d$ is the same for $\xi_1,\ldots,\xi_{p}$; however, this assumption is not needed (see \cite{compositionACC}).

For all $t \ge 0$, the set of \textit{unknown state constraints} is denoted $\SSS_{\rm{u}}(t) \subset \BBR^n$, which is the set of states that satify \textit{a priori} unknown constraints at time $t$. 
For example, $\SSS_{\rm{u}}(t)$ could be the the set of states where a robot is not in collision with any obstacles, where the environment is unmapped. 
We emphasize that $\SSS_{\rm{u}}(t)$ is not assumed to be known \textit{a priori}. 

Next, we define the set of all state constraints
\begin{equation} \label{eq:S_s}
    \SSS_\rms(t) \triangleq \SSS_\rmu(t) \cap \SSS_\xi,
\end{equation}
which is the set of states that are in the known set $\SSS_\xi$ and the unknown set $\SSS_\rmu(t)$.

Since $\SSS_{\rm{u}}(t)$ is not assumed to be known \textit{a priori}, we assume that a real-time sensing system provides a subset of the $\SSS_{\rm{u}}(t)$ at update times $0,T,2T,3T,\ldots$, where $T > 0$.
Specifically, for all $k \in \BBN$, we obtain perception feedback at time $kT$ in the form of a continuously differentiable function $b_k : \mathbb{R}^n \to \mathbb{R}$ such that its zero-superlevel set
\begin{equation}\label{eq:Sk}
\SSS_{k} \triangleq \{x \in \mathbb{R}^n : b_k(x) \ge 0\},
\end{equation}
is nonempty, contains no isolated points, and is a subset of $\SSS_\rmu(kT)$.
We note that different approaches (e.g., \cite{srinivasan2020synthesis, khan2022gaussian,safari2024ACC,safari2024TSCT}) can be used to synthesize the perception feedback function $b_k$. 
The example in \Cref{section:ground robot} uses the simple approach in \cite{safari2024ACC,safari2024TSCT} to synthesize $b_k$ from LiDAR using the soft minimum.

The perception feedback function $b_k$ is required to satify several conditions.
Specifically, we assume that there exist known positive integers $N$ and $r$ such that for all $k \in \mathbb{N}$, the following hold:

\begin{enumerate}[leftmargin=0.9cm]
	\renewcommand{\labelenumi}{(A\arabic{enumi})}
	\renewcommand{\theenumi}{(A\arabic{enumi})}
 \setcounter{enumi}{1}

\item\label{con1}
For all $t \in [kT, (k+N+1)T]$, $\SSS_{k} \subseteq \SSS_{\rm{u}}(t)$.

\item\label{con3}
For all $x \in \BBR^n$, $L_gb_k(x)=L_gL_fb_k(x)=\cdots=L_gL_f^{r-2}b_k(x)=0$.

\item\label{con4}
For almost all $x \in \cup_{i=k}^{k+N} \SSS_{i}$, $L_gL_f^{r-1}b_k(x) \ne 0$.

\item\label{con2}
$\SSS_{k} \cap \left ( \cup_{i=k+1}^{k+N} \SSS_i \right )$ is nonempty and contains no isolated points.

\end{enumerate}

Assumption~\ref{con1} implies that $\SSS_k$ is a subset of the unknown set $\SSS_\rmu(t)$ for $(N+1)T$ time units into the future. 
For example, if $N=1$, then \ref{con1} implies that $\SSS_k$ is a subset of $\SSS_\rmu(t)$ over the interval $[kT,(k+2)T]$. 
The choice $N=1$ is appropriate if $\SSS_\rmu(t)$ is changing quickly. 
On the other hand, if $\SSS_\rmu(t)$ is changing more slowly, then $N>1$ may be appropriate. 
If $\SSS_\rmu(t)$ is time-invariant, then \ref{con1} is satisfied for any positive integer $N$ because $\SSS_k \subseteq \SSS_\rmu$. 

Assumptions~\ref{con3} and~\ref{con4} implies that $b_k$ has relative degree $r$ with respect to \eqref{eq:affine control} on $\cup_{i=k}^{k+N} \SSS_{i}$.
Assumption~\ref{con2} is a technical condition on the perception data that the zero-superlevel set of $b_{k}$ is connected to the union of the zero-superlevel sets of $b_{k-1},\ldots,b_{k-N}$.

Next, let $u_{\rm{d}} : [0, \infty) \times \mathbb{R}^n \to \mathbb{R}^m$ be the \textit{desired control}, which is designed to meet performance requirements but may not account for the state or input constraints.
The objective is to design a full-state feedback control that is as close as possible to $u_{\rm{d}}$ while satisfying the state constraints (i.e., $x(t) \in  \SSS_\rms(t)$) and the input constraints (i.e., $u(t) \in  \SU$). 
Specifically, the objective is to design a full-state feedback control such that the following objectives are satisfied: 
\begin{enumerate}[leftmargin=0.9cm]
	\renewcommand{\labelenumi}{(O\arabic{enumi})}
	\renewcommand{\theenumi}{(O\arabic{enumi})}

\item\label{obj1}
For all $t \geq 0$, $x(t) \in  \SSS_\rms(t)$. 

\item\label{obj2}
For all $t \geq 0$, $u(t) \in  \SU$. 

\item\label{obj3}
For all $t \geq 0$, $\| u(t) - u_{\rm{d}}(t, x) \|$ is small.

\end{enumerate}

All statements in this paper that involve the subscript $k$ are for all $k \in \mathbb{N}$.

\section{Time-Varying Perception Barrier Function} 
\label{sec:Method}

This section is based on the approach in \cite{safari2024TSCT} for constructing a time-varying barrier function from the real-time perception feedback $b_k$. 


Let $\eta:\mathbb{R} \to [0,1]$ be $r+1$-times continuously differentiable such that the following conditions hold:

\begin{enumerate}[leftmargin=0.9cm]
	\renewcommand{\labelenumi}{(C\arabic{enumi})}
	\renewcommand{\theenumi}{(C\arabic{enumi})}

 \item\label{con:con1_g}
For all $t\in (-\infty,0]$, $\eta(t) = 0$.

\item\label{con:con2_g}
For all $t\in [1,\infty)$, $\eta(t) = 1$.

\item\label{con: con4_g}
For all $i \in \{1,\cdots,r+1\}$, $\left . {\frac{{\rm d}^i \eta(t)}{{\rm d}t^i}}\right |_{t=0} = 0$ and $\left . {\frac{{\rm d}^i \eta(t)}{{\rm d}t^i}} \right |_{t=1} = 0$. 
\end{enumerate}

\begin{figure}[t!]
\center{\includegraphics[width=0.47\textwidth,clip=true,trim= 0.42in 0.25in 1.1in 0.6in] {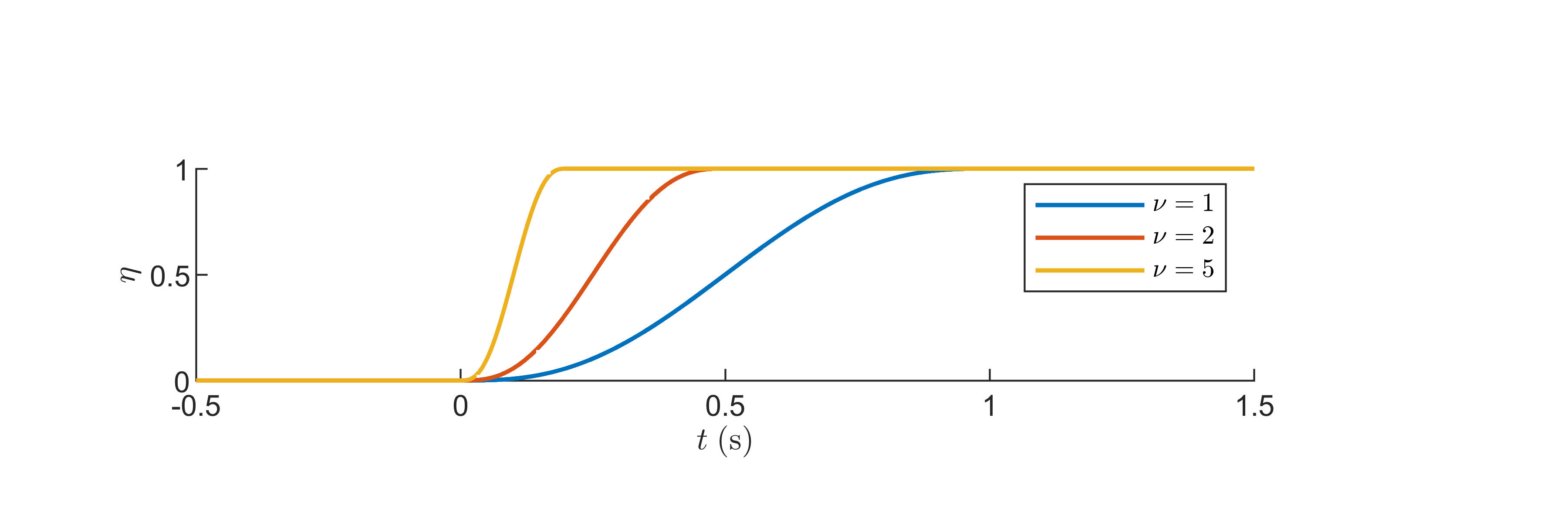}}
\caption{$\eta$ given by Example~\ref{ex:g} with $r=2$.}\label{fig:eta}
\end{figure} 

The following example provides a choice for $\eta$ that satisfies \ref{con:con1_g}--\ref{con: con4_g} for any positive integer $r$.

\begin{example}\label{ex:g}\rm
Let $\nu \ge 1$, and let
\begin{equation}\label{eq:smoothstep}
\eta(t) =
\begin{cases}
        0, & t < 0, \\
        \left(\nu t\right)^{r+2}  \sum_{j=0}^{r+1} \binom{r+1+j}{j}\binom{2r+3}{r+1-j}(-\nu t)^j, & t \in \left [ 0,\frac{1}{\nu} \right ], \\
        1, & t > \frac{1}{\nu},
\end{cases}
\end{equation}
which satisfies \ref{con:con1_g}--\ref{con: con4_g}.
Note that $\nu$ is a tuning parameter, which influences the rate at which $\eta$ transitions from zero to one. 
\Cref{fig:eta} is a plot of $\eta$ for $r=2$. 
\end{example}

Let $\kappa > 0$, and consider $\psi_0:[0, \infty) \times \mathbb{R}^n  \to \mathbb{R}$ such that for all $k\in\BBN$ and all $(t,x) \in [kT, (k+1)T) \times \BBR^n$,
\begin{align}\label{eq:softmax_h}
\psi_0(t,x) &\triangleq \mbox{softmax}_\kappa \bigg (b_{k-1}(x), \cdots, b_{k-N+1}(x),\nonumber\\
&\qquad \eta\left(\textstyle\frac{t}{T}-k\right)b_k(x)  + \left[ 1-\eta\left(\textstyle\frac{t}{T}-k\right)\right]b_{k-N}(x) \bigg ),
\end{align}
where for $i\in \{1,2,\ldots,N\},$ $b_{-i}\triangleq b_0$.
The function $\psi_0$ is constructed from the $N$ most recently obtained perception feedback functions $b_k,\ldots,b_{k-N}$. 
If $N=1$, then 
\begin{equation}
\psi_0(t,x) = \eta\left(\textstyle\frac{t}{T}-k\right)b_k(x)  + \left[ 1-\eta\left(\textstyle\frac{t}{T}-k\right)\right ] b_{k-1}(x).    \label{eq:psi0_N=1}
\end{equation}
In this case, $\psi_0$ is a convex combination of $b_k$ and $b_{k-1}$, where $\psi_0$ smoothly transitions from $b_{k-1}$ to $b_{k}$ over the interval $[kT,(k+1)T]$. 
The final argument of \eqref{eq:softmax_h} involves the convex combination of $b_k$ and $b_{k-N}$, which allows for the smooth transition from $b_{k-N}$ to $b_{k}$ over the interval $[kT,(k+1)T]$. 
This convex combination is the mechanism by which the newest perception feedback $b_{k}$ is smoothly incorporated into $\psi_0$ and the oldest perception feedback $b_{k-N}$ is smoothly removed from $\psi_0$.

The zero-superlevel set of $\psi_0$ is defined by
\begin{equation} \label{eq:safe set final}
    \SSS_\psi(t) \triangleq \{x \in \mathbb{R}^n \colon \psi_0(t,x) \geq 0 \}.
\end{equation}
\Cref{prop:softmin_softmax_sets} implies that at sample times, $\SSS_\psi(t)$ is a subset of the union of the zero-superlevel sets of $b_{k-1},\ldots,b_{k-N}$.
%
%
%
\Cref{prop:softmin_softmax_sets} also implies that for sufficiently large $\kappa>0$, $\SSS_\psi(t)$ approximates the union of the zero-superlevel sets of $b_{k-1},\ldots,b_{k-N}$.
In other words, if $\kappa >0$ is sufficiently large, then $\psi_0$ is a lower-bound approximation of 
\begin{align*}
\psi_*(t,x) &\triangleq \max \, \{ b_{k-1}(x), \cdots, b_{k-N+1}(x),  \nn\\
&\qquad \eta\left(\textstyle\frac{t}{T}-k\right)b_k(x)  + \left[ 1-\eta\left(\textstyle\frac{t}{T}-k\right)\right]b_{k-N}(x)\}.
\end{align*}
However, if $\kappa >0$ is large, then $\textstyle\frac{\partial \psi_0(t,x)}{\partial x}$ has large magnitude at points where $\psi_*$ is not differentiable. 
Thus, choice of $\kappa$ is a trade off between the magnitude of $\textstyle\frac{\partial \psi_0(t,x)}{\partial x}$ and the how well $\SSS_\psi(t)$ approximates the zero-superlevel set of $\psi_*$.

The convex combination of $b_k$ and $b_{k-N}$ used in \eqref{eq:softmax_h} ensures that $\SSS_\psi(t)$ is a subset of $\SSS_{\rm{u}}(t)$ not only at sample times but also for all time between samples. 
The following result demonstrates this property. 
See \cite[Prop. 5]{safari2024TSCT} for a proof. 

\begin{proposition}\label{fact:S(t)}\rm
Assume \ref{con1} is satisfied. 
Then, for all $k \in \BBN$ and all $t \in [kT, (k+1)T]$, $\SSS_\psi(t) \subseteq  \cup_{i=k-N}^k \SSS_i  \subseteq \SSS_{\rm{u}}(t)$. 
\end{proposition}

Next, define
\begin{equation}\label{eq:S}
\SSS(t) \triangleq \SSS_\xi \cap \SSS_\psi(t).
\end{equation}

Since \Cref{fact:S(t)} implies that for all $t \ge 0$, $\SSS_\psi(t) \subseteq \SSS_u(t)$, it follows that for all $t \ge 0$, $\SSS(t) \subseteq \SSS_\rms(t)$. 

It follows from \cite[Prop.~6]{safari2024TSCT} that if $L_gL_f^{r-1}\psi_0$ is nonzero, then $\psi_0$ has relative degree $r$ with respect to \eqref{eq:affine control}.
This result also shows that $L_gL_f^{r-1}\psi_0$ is a convex combination of $L_gL_f^{r-1}b_k,\ldots,L_gL_f^{r-1}b_{k-N}$, which are nonzero from \ref{con4}.

The next section constructs a single time-varying relaxed CBF that composes the input constraints (e.g., $\phi_1\ldots,\phi_\ell$), the \textit{a priori} known state constraints (e.g., $\xi_1, \ldots, \xi_{p}$), and the time-varying barrier function $\psi_0$, which is constructed from the real-time perception feedback $b_k$ using \eqref{eq:softmax_h}.

\section{Control Dynamics and Soft-Minimum Relaxed~CBF}

In order to address both state and input constraints (i.e., \ref{obj1} and \ref{obj2}), we adopt the method in \cite{compositionACC}. 
This approach uses control dynamics to transform input constraints into controller-state constraints, and uses a soft-minimum function to compose multiple candidate CBFs (one for each state and input constraint) into a single relaxed CBF.

\subsection{Control dynamics}

Consider a control $u$ that satisfies the linear time-invariant (LTI) dynamics 
\begin{equation}\label{eq:dynamics_control.a}
    \dot u(t) = A_\rmc u(t) + B_\rmc v(t),
\end{equation}
where $A_\rmc \in \BBR^{m \times m}$ is asymptotically stable, $B_\rmc \in \BBR^{m \times m}$ is nonsingular, $u(0)=u_0 \in \BBR^m$ is the initial condition, and $v: [0, \infty) \to \BBR^m$ is the \textit{surrogate control}, that is, the input to the control dynamics.

The cascade of~\cref{eq:affine control,eq:dynamics_control.a} is
\begin{equation}\label{eq:dynamics_aug.1}
    \dot{\tilde x} = \tilde f(\tilde x) + \tilde g(\tilde x) v(\tilde x),
\end{equation}
where
\begin{gather}\label{eq:dynamics_aug.2}
    \tilde x \triangleq  \begin{bmatrix}
            x \\ u
        \end{bmatrix}, \qquad 
       \tilde f( \tilde x) \triangleq \begin{bmatrix}
            f(x) + g(x) u\\
            A_\rmc u
        \end{bmatrix}, \qquad 
        \tilde g \triangleq \begin{bmatrix}
            0\\
            B_\rmc
        \end{bmatrix}, 
\end{gather}
and $\tilde x_0 \triangleq \begin{bmatrix} x_0^\rmT& u_0^\rmT \end{bmatrix}^\rmT$.
Define 
\begin{equation}\label{eq:hat S}
\tilde \SSS(t) \triangleq \SSS(t) \times \SU,
\end{equation}
which is the set of cascade states $\tilde x$ such that the state constraint (i.e., $x(t) \in \SSS(t)$) and the input constraint (i.e., $u(t) \in \SU$) are satisfied. 


The control dynamics \Cref{eq:dynamics_control.a} transform the input constraints into controller-state constraints. 
However, the control dynamics increase the relative degree of the state constraints $\psi_0,\xi_1, \dots, \xi_{p}$.
It follows from \cite[Proposition~1]{compositionACC} that the control dynamics \Cref{eq:dynamics_control.a} increase the relative degree of the state constraints by one. 
Therefore, the relative degree of $\psi_0$ with respect to the cascade \Cref{eq:dynamics_aug.1,eq:dynamics_aug.2} is $r + 1$ and the relative degree of $\xi_1, \dots, \xi_{p}$ with respect to the cascade \Cref{eq:dynamics_aug.1,eq:dynamics_aug.2} is $d + 1$.
We also note that for all $j \in \{1, 2, \dots, \ell\}$, $\phi_1(u)\ldots,\phi_\ell(u)$ has relative degree one with respect to the cascade \Cref{eq:dynamics_aug.1,eq:dynamics_aug.2}.

\subsection{Composite Soft-Minimum Relaxed CBF}

Since $\psi_0$ and $\xi_1, \dots, \xi_{p}$ have relative degree greater than one, we use a higher-order approach to construct a higher-order candidate CBFs. 

First, for all $i \in \{0, 1, \cdots, r-1\}$, let $\alpha_{\psi, i}:\mathbb{R} \to \mathbb{R}$ be an $(r -i)$-times continuously differentiable extended class-$\mathcal{K}$ function, 
and consider $\psi_i:[0, \infty) \times \mathbb{R}^{n+m}  \to \mathbb{R}$ defined by
\begin{equation}\label{eq:HOCBF.psi}
\psi_{i+1}(t,\tilde x) \triangleq \textstyle\frac{\partial \psi_{i}(t,\tilde x)}{\partial t} + L_{\tilde f} \psi_{i}(t,\tilde x) +\alpha_{\psi, i}(\psi_{i}(t,\tilde x)).
\end{equation}


Similarly, for all $i \in \{0, 1, \ldots, d-1 \}$ and all $j \in \{1, 2, \ldots, p \}$, let $\alpha_{\xi, j, i}:\mathbb{R} \to \mathbb{R}$ be an $(d -i)$-times continuously differentiable extended class-$\mathcal{K}$ function, 
and consider $\xi_{j,i}:\mathbb{R}^{n+m}  \to \mathbb{R}$ defined by
\begin{equation}\label{eq:HOCBF.varphi}
\xi_{j, i+1}(\tilde x) \triangleq  L_{\tilde f} \xi_{j, i}(\tilde x) +\alpha_{\xi, j, i}(\xi_{j, i}(\tilde x)).
\end{equation}

Next, define
\begin{align}\label{eq:common set varphi.varpsi.u}
\bar \SC(t) &\triangleq \{ \tilde x \in \BBR^{n+m} \colon \psi_0(t,\tilde x) \ge 0,\ldots,\psi_{r}(t,\tilde x) \ge 0, \nn\\
&\qquad \xi_{1,0}(\tilde x) \ge 0,\ldots,\xi_{1,d}(\tilde x) \ge 0,\ldots, \nn\\
&\qquad \xi_{p,0}(\tilde x) \ge 0,\ldots,\xi_{p,d}(\tilde x) \ge 0,\nn\\
&\qquad \phi_{1}(\tilde x) \ge 0,\ldots,\phi_{\ell}(\tilde x) \ge 0 \},
\end{align}
and note that $\bar \SC(t) \subseteq \tilde \SSS (t) = \SSS(t) \times \SU$. 
We also define
\begin{align*}
    \bar \SH(t) &\triangleq \{ \tilde x \in \BBR^{n+m} \colon \psi_{r}(t,\tilde x) \ge 0, \xi_{1,d}(\tilde x) \ge 0,\ldots, \\
&\qquad \xi_{p,d}(\tilde x) \ge 0,\phi_{1}(\tilde x) \ge 0,\ldots,\phi_{\ell}(\tilde x) \ge 0 \},
\end{align*}
It follows from \cite[Theorem~3]{xiao2021high} that if $\tilde x_0 \in \bar \SC(0)$ and for all $t \geq 0$, $\tilde x(t) \in \bar \SH(t)$, then for all $t \geq 0$, $\tilde x(t) \in \bar \SC(t)$. 
Thus, we consider a candidate CBF whose zero-superlevel set is a subset of $\bar \SH(t)$. 
Specifically, let $\varepsilon > 0$ and consider $h:[0, \infty) \times \BBR^{n+m} \to \BBR$ defined by 
\begin{align}\label{eq:softmin h}
    h(t, \tilde x) &\triangleq \mbox{softmin}_\varepsilon \Big( \psi_{r}(t, \tilde x), \xi_{1,d}(\tilde x), \cdots, \xi_{p,d}(\tilde x), \nn \\ &\qquad   \phi_{1}(\tilde x), \cdots, \phi_{\ell}(\tilde x) \Big).   
\end{align}
The zero-superlevel set of $h$ is 
\begin{equation*} 
    \SH(t) \triangleq \{\tilde x \in \BBR^{n+m} \colon h(t, \tilde x) \geq 0 \}.
\end{equation*}
\Cref{fact:softmin_limit} implies that $\SH(t) \subseteq \bar \SH(t)$. 
In fact, \Cref{fact:softmin_limit} shows that as $\epsilon \to \infty$, $\SH(t) \to \bar \SH(t)$.
In other words, for sufficiently large $\epsilon >0$, $\SH(t)$ approximates $\bar \SH(t)$. 

Next, define 
\begin{equation*}
    \SB(t) \triangleq \Big \{ \tilde x \in \bd \SH(t) \colon \textstyle\frac{\partial h(t,\tilde x)}{\partial t} + L_{\tilde f} h(t,\tilde x) \leq 0 \Big\}, 
\end{equation*}
which is the set of all states on the boundary of the zero-superlevel set of $h$ such that if $v=0$, then the time derivative of $h$ is nonpositive along the trajectories of \Cref{eq:dynamics_aug.1,eq:dynamics_aug.2}.
We assume that on $v$ directly impacts the time derivative of $h$ on $\SB$. 
Specifically, we make the following assumption:
\begin{enumerate}[leftmargin=0.9cm]
\renewcommand{\labelenumi}{(A\arabic{enumi})}
\renewcommand{\theenumi}{(A\arabic{enumi})}
\setcounter{enumi}{5}
\item\label{con6}
For all $(t,\tilde x) \in [0,\infty) \times \SB(t)$, $L_{\tilde g}h(t,\tilde x) \ne 0$.
\end{enumerate}
Assumption \ref{con6} is related to the constant-relative-degree assumption often invoked with CBF approaches. 
In this work, $L_{\tilde g} h$ is assumed to be nonzero on $\SB$, which is a subset of the boundary of the zero-superlevel set of $h$.
The next result is from \cite[Proposition~8]{safari2024TSCT} and demonstrates that $h$ is a relaxed CBF in the sense that it satisfies the CBF condition on $\SB$.

\begin{proposition}\label{prop:CFI}
\rm
Assume \ref{con6} is satisfied. 
Then, for all $(t,\tilde x) \in [0,\infty) \times  \mbox{bd }\SH(t)$
\begin{equation}
\sup_{v \in \BBR^m} \left [ \textstyle\frac{\partial h(t,\tilde x)}{\partial t} + L_{\tilde f} h(t,\tilde x) + L_{\tilde g} h(t,\tilde x) v \right ]  \ge 0.\label{def:RCBF.1}
\end{equation} 
\end{proposition}

Since $h$ is a relaxed CBF, Nagumo's theorem \cite[Corollary~4.8]{blanchini2008set} suggests that there exist a control such that $\SH(t)$ is forward invariant with respect to the \Cref{eq:dynamics_aug.1,eq:dynamics_aug.2}. 
However, $\SH(t)$ is not necessarily a subset of the $\tilde \SSS(t)$.
Therefore, the next result, which follows from \cite[Proposition~9]{safari2024TSCT}, is useful because it shows that forward invariance of $\SH(t)$ implies forward invariance of 
\begin{equation*}
    \SC(t) \triangleq \SH(t) \cap \bar \SC(t),
\end{equation*}
which is a subset of the $\tilde \SSS(t)$.

\begin{proposition}\label{prop:Forward_Invariant}
\rm
Consider \Cref{eq:affine control,eq:dynamics_aug.1,eq:dynamics_aug.2}, where \Cref{con3,con1_varphi} are satisfied and $\tilde x_0 \in \SC(0)$. 
Assume there exists $\bar{t} \in (0,\infty]$ such that for all $t \in [0,\bar{t})$, $\tilde x(t) \in \SH(t)$.
Then, for all $t \in [0,\bar{t})$, $\tilde x(t) \in \SC(t)$.
\end{proposition}

\section{Optimal Control with State and Input Constraints}

This section uses the relaxed CBF $h$ to construct a closed-form control that satisfies state and input constraints. 
We note that the control $u$ is generated from the LTI control dynamics \Cref{eq:dynamics_control.a}, where the surrogate control $v$ is the input. 
Thus, this section aims to construct a surrogate control such that the state and input constraints are satisfied and \ref{obj3} is accomplished, that is, the control $u$ is as closed as possible to the desired control $u_\rmd$. 
To address \ref{obj3}, we consider the \textit{desired surrogate control} $v_\rmd: [0, \infty) \times \BBR^{n+m} \to \BBR^m$ defined by
\begin{equation}\label{eq:u_d_hat_def}
    v_\rmd(t, \tilde x) \triangleq B_\rmc^{-1} \Big( L_{f} u_\rmd(t, x)- A_\rmc+\sigma \left( u_\rmd(t, x)- u(t) \right) \Big),
\end{equation}
where $\sigma > 0$. 
Proposition~2 in \cite{compositionACC} shows that if $v = v_\rmd$, then the control $u$ converges exponentially to the desired control $u_\rmd$.
Thus, we aim to generate a surrogate control $v$ that is as close as possible to $v_\rmd$ subject to a CBF-based constraint that ensures state and input constraints are satisfied.

Consider the constraint function $b \colon [0,\infty) \times \BBR^{n+m} \times \BBR^m \times \BBR \to \BBR$ given by
\begin{align}
b(t,\tilde x,\hat{v},\hat{\mu}) &\triangleq \textstyle\frac{\partial h(t,\tilde x)}{\partial t} + L_{\tilde f} h(t,\tilde x) + L_{\tilde g}h(t,\tilde x)\hat{v} \nn \\
        & \qquad +\alpha(h(t,\tilde x)) + \hat{\mu}h(t,\tilde x),
        \label{eq:safety_constraint}
\end{align}
where $\hat v$ is the control variable; $\hat \mu$ is a slack variable; and $\alpha \colon \BBR \to \BBR$ is locally Lipschitz and nondecreasing such that $\alpha(0)=0$. 
Next, let $\gamma > 0$, and consider the cost function  $\SJ \colon [0,\infty) \times \BBR^{n+m} \times \BBR^m \times \BBR \to \BBR$ given by
\begin{equation}\label{eq:SJ}
 \SJ(t, \tilde x, \hat{v}, \hat{\mu}) \triangleq \frac{1}{2}\| \hat v -v_{\rm{d}}(t, \tilde x) \| ^2 + \frac{\gamma}{2} \hat{\mu}^2, 
\end{equation} 
The objective is to synthesize $(\hat v,\hat \mu)$ that minimizes the cost $\SJ(t,\tilde x,\hat{v},\hat{\mu})$ subject to the relaxed CBF safety constraint $b(t,\tilde x,\hat{v},\hat{\mu}) \ge 0$.

For each $(t,\tilde x) \in [0,\infty) \times \BBR^{n+m}$, the minimizer of $\SJ(t,\tilde x,\hat v,\hat \mu)$ subject to $b(t,\tilde x,\hat{v},\hat{\mu}) \ge 0$ can be obtained from the first-order necessary conditions for optimality.
For example, see \cite{wieland2007constructive,ames2016control}.
The minimizer of $\SJ(t,\tilde x,\hat v,\hat \mu)$ subject to $b(t,\tilde x,\hat{v},\hat{\mu}) \ge 0$ is the control $v_* \colon [0, \infty) \times \mathbb{R}^{n+m}$ and slack variable $\mu_* \colon [0,\infty) \times \BBR^{n+m} \to \BBR$ given by
\begin{gather}
v_* (t,\tilde x) \triangleq v_d(t, \tilde x) + \lambda(t,\tilde x) L_{\tilde g}h(t,\tilde x)^\rmT , \label{eq:uclose}  \\ 
\mu_*(t,\tilde x) \triangleq  \gamma^{-1} h(t,\tilde x) \lambda(t,\tilde x), \label{eq:mu_close}
\end{gather}
where $\lambda,\omega \colon [0, \infty) \times \BBR^{n+m} \to \BBR$ are defined by 
\begin{align}
        \lambda(t,\tilde x) &\triangleq \begin{cases}
            \tfrac{-\omega(t,\tilde x)}{\| L_{\tilde g} h(t,\tilde x)^\rmT \|^2 + \gamma^{-1}h(t,\tilde x)^2},& \omega(t,\tilde x) < 0,\\
            0, & \omega(t,\tilde x) \ge 0,
        \end{cases} \label{eq:ulambda} \\
\omega(t,\tilde x) &\triangleq b(t,\tilde x,v_\rmd(t,\tilde x),0). \label{eq:omegabar}
\end{align}


The next result demonstrates that $(v_*(t, \tilde x), \mu_*(t, \tilde x))$ is the unique global minimizer of $\SJ(t,\tilde x,\hat v,\hat \mu)$ subject to $b(t,\tilde x,\hat{v},\hat{\mu}) \ge 0$.

\begin{theorem}\label{thm:global minimizer}
\rm
Assume \ref{con6} is satisfied. 
Let $t \ge 0$ and $\tilde x \in \BBR^{n+m}$.
Furthermore, let $v \in \BBR^m$ and $\mu \in \BBR$ be such that $b(t,\tilde x,v,\mu) \ge 0$ and $(v,\mu) \ne (v_*(t,\tilde x),\mu_*(t,\tilde x))$.
Then, 
\begin{equation} \label{eq:Global_Min}
    \SJ(t,\tilde x,v,\mu) > \SJ(t,\tilde x,v_*(t,\tilde x),\mu_*(t,\tilde x)).
\end{equation}
\end{theorem}
\begin{proof}
Define $\textstyle{\Delta \SJ \triangleq \SJ(t, \tilde{x}, v, \mu) - \SJ(t, \tilde{x}, v_*(t, \tilde{x}), \mu_*(t, \tilde{x}))}$, $Q \triangleq I_m$, and $c \triangleq -v_d(t, \tilde{x})$. Using \Cref{eq:safety_constraint,eq:uclose,eq:mu_close,eq:ulambda} and direct computation, it follows that
\begin{align*}
\Delta \SJ &= \frac{1}{2}v^TQv + c^Tv + \frac{\gamma}{2}\mu^2 - \frac{1}{2}v_*^TQv_* - c^Tv_* - \frac{\gamma}{2}\mu_*^2\\
&= \frac{1}{2}(v - v_*)^TQ(v - v_*) + \frac{\gamma}{2}(\mu - \mu_*)^2 + \lambda b(t, \tilde{x}, v, \mu),
\end{align*}
where $(t, \tilde{x})$ are omitted for brevity. Since $\lambda \geq 0$ and $b(t, \tilde{x}, v, \mu) \geq 0$, it follows that
\begin{align}
\Delta \SJ \geq \frac{1}{2}(v - v_*)^T(v - v_*) + \frac{\gamma}{2}(\mu - \mu_*)^2.
\end{align}
Since $\gamma > 0$, and $(v, \mu) \neq (v_*, \mu_*)$, it follows that $\Delta \SJ > 0$, which confirms \eqref{eq:Global_Min}.
\end{proof}
The following theorem is the main result of satisfaction of state and input constraints. 
It demonstrate that the control makes $\SC(t) \subseteq \tilde \SSS (t) = \SSS(t) \times \SU$ forward invariant. 

\begin{theorem}\label{Th:Main th}
\rm
Consider \eqref{eq:affine control}, where \ref{con1_varphi}--\ref{con6} are satisfied. 
Let $u$ be given by \Cref{eq:dynamics_control.a}, where $v=v_*$ and $v_*$ is given by~\Cref{eq:uclose,eq:ulambda,eq:omegabar}.
Assume that $h^\prime$ is locally Lipschitz in $\tilde x$ on $\BBR^{n+m}$.
Then, for all $\tilde x_0 \in \SC(0)$, the following statements hold: 
\begin{enumerate}
\item \label{main.thm.1}
There exists a maximum value $t_{\rm m}(\tilde x_0) \in (0,\infty]$ such that \Cref{eq:dynamics_aug.1,eq:dynamics_aug.2} with $v=v_*$ has a unique solution on $[0, t_{\rm m}(\tilde x_0))$. 

\item \label{main.thm.2}
For all $t \in [0, t_{\rm m}(\tilde x_0))$, $\tilde x(t) \in \SC(t) \subseteq \tilde \SSS(t)$.

\end{enumerate}
\end{theorem}
\begin{proof}
To prove \ref{main.thm.1}, it follows from \Cref{eq:uclose,eq:omegabar,eq:ulambda} that $v_*$ is continuous in $t$ on $[0, \infty)$ and locally Lipschitz in $\tilde{x}$ on $\mathbb{R}^{n+m}$. Since, in addition, $\tilde{f}$ and $\tilde{g}$ are locally Lipschitz on $\mathbb{R}^{n+m}$, it follows from \cite[Theorem 3.1]{khalil2002control} that \Cref{eq:dynamics_aug.1,eq:dynamics_aug.2} with $v = v_*$ has a unique solution on $[0, t_m(\tilde{x}_0))$.

To prove \ref{main.thm.2}, it follows from \Cref{eq:dynamics_aug.1,eq:dynamics_aug.2} with $v = v_*$ that
\begin{align}
\dot{\bar{x}} = \bar{f}(\bar{x}),
\end{align}
where
\begin{align}\label{eq:taug_dyn}
\bar{x} \triangleq \begin{bmatrix} t \\ \tilde{x} \end{bmatrix}, \quad
\bar{f}(\bar{x}) \triangleq \begin{bmatrix} 1 \\ \tilde{f}(\tilde{x}) + \tilde{g}(\tilde{x})v_*(t, \tilde{x}) \end{bmatrix}, \quad \bar{x}_0 \triangleq \begin{bmatrix} 0 \\ \tilde{x}_0 \end{bmatrix}.
\end{align}

Next, define
\begin{align}\label{eq:taug_h}
\bar{h}(\bar{x}) \triangleq h(t, \tilde{x}),
\end{align}
\begin{align}\label{eq:taug_set}
\bar{\mathcal{X}} \triangleq \{\bar{x} \in [0, \infty) \times \mathbb{R}^{n+m} : \bar{h}(\bar{x}) \geq 0\},
\end{align}
and it follows from \Cref{eq:taug_dyn,eq:taug_h} that
\begin{align}\label{eq:taug_lfh}
L_{\bar{f}}\bar{h}(\bar{x}) = \frac{\partial h(t,\tilde{x})}{\partial t} + L_{\tilde{f}}h(t, \tilde{x}) + L_{\tilde{g}}h(t, \tilde{x})v_*(t, \tilde{x}).
\end{align}

Let $\bar{x}_1 \in \text{bd}\;\bar{\mathcal{X}}$, and let $(t_1, \tilde{x}_1) \in [0, \infty) \times \mathbb{R}^{n+m}$ be such that $\bar{x}_1^\rmT = [t_1 \; \tilde{x}_1^\rmT]$. Since $\bar{x}_1 \in \text{bd}\;\bar{\mathcal{X}}$, it follows from \Cref{eq:taug_h,eq:taug_set} that $h(t_1, \tilde{x}_1) = 0$. Thus, substituting \Cref{eq:uclose} into \Cref{eq:taug_lfh}, evaluating at $\bar{x}_1$, and using \Cref{eq:ulambda,eq:omegabar} yields
\begin{align}
L_{\bar{f}}\bar{h}(\bar{x}_1) =
\begin{cases}
0, & \omega(t_1, \tilde{x}_1) < 0,\\
\omega(t_1, \tilde{x}_1), & \omega(t_1, \tilde{x}_1) \geq 0.
\end{cases}
\end{align}
Thus, for all $\bar{x} \in \text{bd}\;\bar{\mathcal{X}}$, $L_{\bar{f}}\bar{h}(\bar{x}) \geq 0$.

Next, since $\tilde{x}_0 \in C(0) \subseteq \bar \SX(0)$, it follows from \Cref{eq:taug_h,eq:taug_set} that $\bar{x}_0 \in \bar{\mathcal{X}}$. Since, in addition, for all $\bar{x} \in \text{bd}\;\bar{\mathcal{X}}$, $L_{\bar{f}}\bar{h}(\bar{x}) \geq 0$, it follows from Nagumo's theorem \cite[Corollary 4.8]{blanchini2008set} that for all $t \in [0, t_m(\tilde{x}_0))$, $\bar{x}(t) \in \bar{\mathcal{X}}$. Thus, \Cref{eq:taug_h,eq:taug_set} imply that for all $t \in [0, t_m(\tilde{x}_0))$, $\tilde{x}(t) \in \SH(t)$. Therefore, \Cref{prop:Forward_Invariant} implies that for all $t \in [0, t_m(\tilde{x}_0))$, $\tilde{x}(t) \in C(t) \subseteq \tilde{S}(t)$, which yields \ref{main.thm.2}.
\end{proof}

\Cref{Th:Main th} demonstrates that the control 
\Cref{eq:dynamics_control.a,eq:uclose,eq:ulambda,eq:omegabar} with $v=v_*$ yields $\tilde x(t) \in \SC(t) \subseteq \tilde \SSS(t)$, which implies that $x(t) \in \SSS_\rms(t)$ and $u(t) \in \SU$. 
Hence, \ref{obj1} and \ref{obj2} are satisfied. 
Furthermore, $v_*$ is as closed as possible to $v_\rmd$, which is designed to satisfy \ref{obj3}.

\section{Application to a Ground Robot}
\label{section:ground robot}

Consider the nonholonomic ground robot modeled by \eqref{eq:affine control}, where
\begin{equation*}
    f(x) = \begin{bmatrix}
     v\cos\theta \\
     v\sin\theta \\
    0 \\
    0
    \end{bmatrix}, 
    \,
    g(x) = \begin{bmatrix}
    0 & 0\\
    0 & 0\\
    1 & 0 \\
    0 & 1
    \end{bmatrix}, 
    \,
    x = \begin{bmatrix}
    q_\rmx\\
    q_\rmy\\
    v\\
    \theta
    \end{bmatrix}, 
    \,
    u = \begin{bmatrix}
    u_1\\
    u_2
    \end{bmatrix}, 
\end{equation*}
and $q \triangleq [ \, q_\rmx \quad q_\rmy \, ]^\rmT$ is the robot's position in an orthogonal coordinate system, $v$ is the speed, and $\theta$ is the direction of the velocity vector (i.e., the angle from $[ \, 1 \quad 0 \, ]^\rmT$ to $[ \, \dot q_\rmx \quad \dot q_\rmy \, ]^\rmT$).

We consider the control input constraint for the robot. Specifically, the control must remain in the admissible set $\SU$ given by \eqref{eq:U_def}, where 
\begin{gather*}
    \phi_1(u) = 6 - u_1, \quad \phi_2(u) = u_1 + 6, \\ 
    \phi_3(u) = 4 - u_2, \quad \phi_4(u) = u_2 + 4, 
    \end{gather*}
which implies $\SU = \{ u \in \BBR^2 \colon | u_1 | \le 6 \mbox{ and } |u_2| \le 4 \}$ and $m = 4$. Next, for the known state constraints, we consider bounds on speed $v$. Specifically, the bounds on speed $v$ are modeled as the zero-superlevel sets of 
\begin{equation*}
    \xi_1 = 3 - v, \quad \xi_2 = v + 3, 
\end{equation*}
which implies $d = 1$ and $p = 2$. 

Next, for the unknown safe set $\SSS_u(t)$, we assume that the robot is equipped with a perception/sensing system (e.g., LiDAR) with $360^\circ$ field of view (FOV) that detects up to $\bar \ell$ points on objects that are: (i) in line of sight of the robot; (ii) inside the FOV of the perception system; and (iii) inside the detection radius $\bar r > 0$ of the perception system. 
Specifically, for all $k \in \BBN$, at time $t=kT$, the robot obtains raw perception feedback in the form of $\ell_k \in \{0,1,\ldots,\bar \ell \}$ points given by $(r_{1,k},\theta_{1,k}),\cdots,(r_{\ell_k,k},\theta_{\ell_k,k})$, which are the polar-coordinate positions of the detected points relative to the robot position $q(kT)$ at the time of detection.
For all $i\in \{1,2,\ldots,\ell_k \}$, $r_{i,k} \in [0,\bar r]$ and $\theta_{i,k} \in [0,2\pi)$.

To synthesize the perception feedback function $b_k$, we adopt the approach proposed by \cite{safari2024TSCT}. For each detected point, an ellipse is formed, with its semi-major axis extending to the boundary of the perception system's detection area. The interior of each ellipse is considered unsafe, and a soft minimum is used to combine all elliptical functions with the perception detection area. The zero-superlevel set of the resulting composite soft-minimum CBF defines a subset of the unknown safe set $\SSS_u(t)$ at time $kT$.

It follows that for all $i\in \{1,2,\ldots, \ell_k \}$, the location of the detected point is 
\begin{equation*}
   c_{i,k} \triangleq q(kT) + r_{i,k} \begin{bmatrix}
        \cos \theta_{i,k} \\ \sin \theta_{i,k}
    \end{bmatrix}.
\end{equation*}
Similarly, for all $i\in \{1,2,\ldots, \ell_k \}$,
\begin{equation*}
   d_{i,k} \triangleq q(kT) + \bar r \begin{bmatrix}
        \cos \theta_{i,k} \\ \sin \theta_{i,k}
    \end{bmatrix}
\end{equation*}
is the location of the point that is at the boundary of the detection radius and on the line between $q(kT)$ and $c_{i,k}$.

Next, for each point $(r_{i,k},\theta_{i,k})$, we consider a function whose zero-level set is an ellipse that encircles $c_{i,k}$ and $d_{i,k}$.
Specifically, for all $i\in \{1,2,\ldots,\ell_k \}$, consider $\sigma_{i,k} \colon \BBR^4 \to \BBR$ defined by 
\begin{align}
    \sigma_{i,k}(x) &\triangleq \left [ \chi(x) - \frac{1}{2} (c_{i,k}+d_{i,k}) \right ]^\rmT R^\rmT_{i,k} P_{i,k} R_{i,k} \nn\\
    &\qquad \times \left [ \chi(x) - \frac{1}{2} (c_{i,k}+d_{i,k}) \right ] - 1, \label{eq:ellipse}
\end{align}
where 
\begin{gather}
    R_{i,k} \triangleq \begin{bmatrix}
        \cos\theta_{i,k} & \sin\theta_{i,k} \\ -\sin\theta_{i,k}& \cos\theta_{i,k}
    \end{bmatrix}, \quad
    P_{i,k} \triangleq \begin{bmatrix}
        a_{i,k}^{-2}&0 \\ 0 & z_{i,k}^{-2}
    \end{bmatrix}, \label{eq:ellipse.2}\\
    a_{i,k} \triangleq \frac{\bar r-r_{i,k}}{2} + \varepsilon_a, \quad 
    z_{i,k} \triangleq \sqrt{a_{i,k}^2 - \left(\frac{\bar r-r_{i,k}}{2}\right)^2}, \label{eq:ellipse.3}
\end{gather}
where $\chi(x) \triangleq [ \, I_2 \quad 0_{2\times2} \, ]^\rmT x$ extracts robot position from the state, and $\varepsilon_a > 0$ determines the size of the ellipse $\sigma_{i,k}(x)=0$ (specifically, larger $\varepsilon_a$ yields a larger ellipse). 
The parameter $\varepsilon_a$ is used to introduce conservativeness that can, for example, address an environment with dynamic obstacles. 
Note that $a_{i,k}$ and $z_{i,k}$ are the lengths of the semi-major and semi-minor axes, respectively. 
The area outside the ellipse is the zero-superlevel set of $\sigma_{i,k}$.

Next, let $\beta_k \colon \BBR^4 \to \BBR$ be a continuously differentiable function whose zero-superlevel set models the perception system's detection area. For example, for a $360^\circ$ FOV with detection radius $\bar r >0$, $\beta_k$ is defined by
\begin{align}\label{eq:circle}
    \beta_k(x) \triangleq  \left(\bar r - \varepsilon_\beta \right)^2- \| \chi(x)-q(kT) \|^2,
\end{align}
which has a zero-superlevel that is a disk with radius $\bar r - \varepsilon_\beta$ centered at the robot position $q(kT)$ at the time of detection, where $\varepsilon_\beta \ge 0$ influences the size of the disk and plays a role similar to $\varepsilon_a$.

We construct the perception feedback function $b_k$ using the soft minimum to compose $\sigma_{1,k},\cdots,\sigma_{\ell_k,k}$ and $\beta_k$.
Specifically, let $\rho>0$ and consider 
\begin{equation} \label{eq:bk_ex1}
    b_k(x) = \begin{cases}
            \mbox{softmin}_{\rho} \left( \beta_k(x), \sigma_{1,k}(x),\ldots,\sigma_{\ell_k,k}(x)\right), & \ell_k > 0,\\
            \beta_k(x), & \ell_k =0.
        \end{cases}
\end{equation}

The control objective is for the robot to move from its initial location to a goal location $q_\rmg = [ \, q_{\rmg,\rmx} \quad q_{\rmg,\rmy} \, ]^\rmT \in\BBR^2$ without violating safety (i.e., hitting an obstacle).
To accomplish this objective, we consider the desired control
\begin{equation*}
    u_\rmd(x)  \triangleq \begin{bmatrix} u_{\rmd_1}(x)\\u_{\rmd_2}(x) \end{bmatrix}, 
\end{equation*}
 where
\begin{align*}
u_{\rmd_1}(x) &\triangleq  -(k_1+k_3) v + (1+k_1k_3)\| q - q_\rmg \| \cos \delta(x)\\
&\qquad + k_1\left ( k_2 \| q - q_\rmg \| +v \right )\sin^2\delta(x),\\
u_{\rmd_2}(x) &\triangleq \left ( k_2+\frac{v}{\| q - q_\rmg \|} \right )\sin\delta(x),\\
\delta(x) &\triangleq\mbox{atan2}(q_\rmy-q_{\rmg,\rmy},q_\rmx-q_{\rmg,\rmx})-\theta + \pi,
\end{align*}
and $k_1,k_2,k_3 > 0$.
The desired control $u_\rmd$ drives $q$ to $q_\rmg$ but does not account for safety and control input constraints.
The desired control is designed using a process similar to \cite[pp.~30--31]{de2002control}.

For this example, the perception update period is $T = 0.2$~s and the gains for the desired control are $k_1 = 0.2$, $k_2 = 1$, and $k_3 = 2$. For the surrogate control, we use control dynamics $\eqref{eq:dynamics_control.a}$, where $A_\rmc =\matls -1 &  0 \\ 0 & -1 \matrs$ and $B_\rmc = -A_\rmc$, which result in low-pass control dynamics. The desired surrogate control is given by \eqref{eq:u_d_hat_def}, where $\sigma = 0.6$. 

We use the perception feedback $b_k$ given by \eqref{eq:bk_ex1}, where $\rho = 30$, $\bar r = 5$~m, and $\varepsilon_a = \varepsilon_\beta = 0.15$~m. 
The maximum number of detected points is $\bar \ell = 100$, and $N = 3$. \Cref{fig:map} shows a map of the unknown environment that the ground robot aims to navigate. We implement the control~\Cref{eq:softmax_h,eq:dynamics_control.a,eq:u_d_hat_def,eq:HOCBF.varphi,eq:softmin h,eq:uclose,eq:ulambda,eq:omegabar} with $\kappa = 30$, $\varepsilon = 10$, $\gamma = 200$, $\alpha_{\psi, 0}(\psi_0) =25\psi_0$, $\alpha_{\psi, 1}(\psi_1) =20\psi_1$, for all $j \in \{1,\ldots,4\}$,  $\alpha_{\xi,j,0}(\xi_j) =15\xi_j$, $\alpha(h) =30h$, and $\eta$ given by Example~\ref{ex:g} where $r=2$ and $\nu = 1.2$. 
For sample-data implementation, the control is updated at $100$~Hz.

\Cref{fig:map} shows the closed-loop trajectories for $\tilde x_0 = [\,-1\quad -8\quad0\quad\frac{\pi}{2}\quad 0 \quad 0\,]^\rmT$ with 3 goal locations: $q_\rmg = [\,6\quad2.5\,]^\rmT$~m, $q_\rmg = [\,-7\quad-1.5\,]^\rmT$~m, and $q_\rmg = [\,-5\quad7\,]^\rmT$~m. 
In all cases, the robot converges to the goal while satisfying safety and control input constraints. 
\Cref{fig:control_signal,fig:h_vals} show time histories of the relevant signals for the case where $q_\rmg = [\,6\quad2.5\,]^\rmT$~m. 
\Cref{fig:h_vals} shows $h$, $\psi_0$, $\min \xi_j$, and $\min \phi_j$ are positive for all time, which demonstrates that for all time $t$, x remains in $\SSS_\rms(t)$ and $u \in \SU$.   
\Cref{fig:control_signal} shows $v$ deviates from $v_\rmd$ in order to satisfy safety. 

\begin{figure}[ht!]
\center{\includegraphics[width=0.46\textwidth,clip=true,trim= 0.1in 0.3in 0.1in 0.7in] {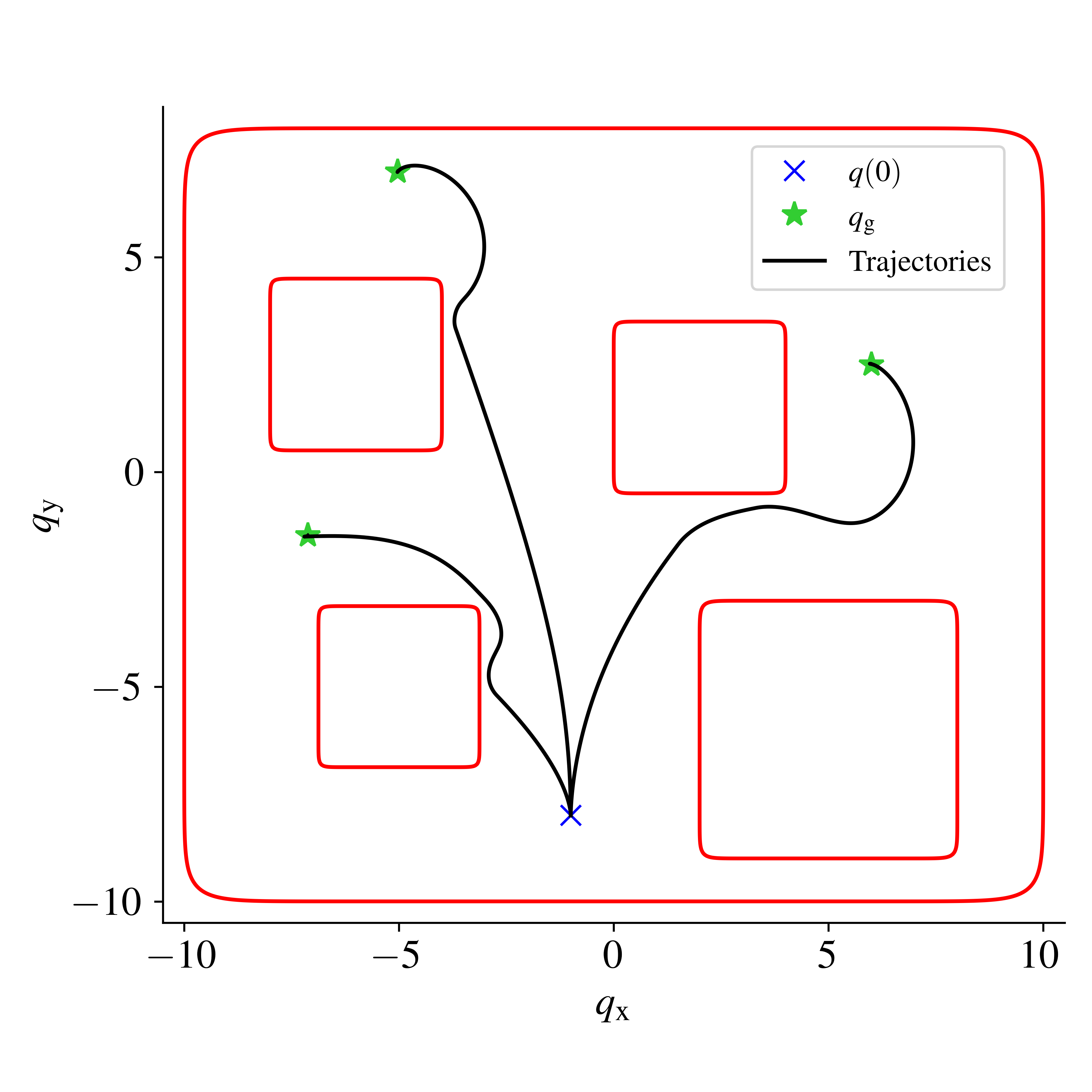}}
\caption{Three closed-loop trajectories using the control~\Cref{eq:softmax_h,eq:dynamics_control.a,eq:u_d_hat_def,eq:HOCBF.varphi,eq:softmin h,eq:uclose,eq:ulambda,eq:omegabar} with the perception feedback $b_k$ generated from $360^{\circ}$ FOV perception in a static environment.}
\label{fig:map}
\end{figure} 

\begin{figure}[ht!]
\center{\includegraphics[width=0.46\textwidth,clip=true,trim= 0.15in 0.1in 0.1in 0.15in] {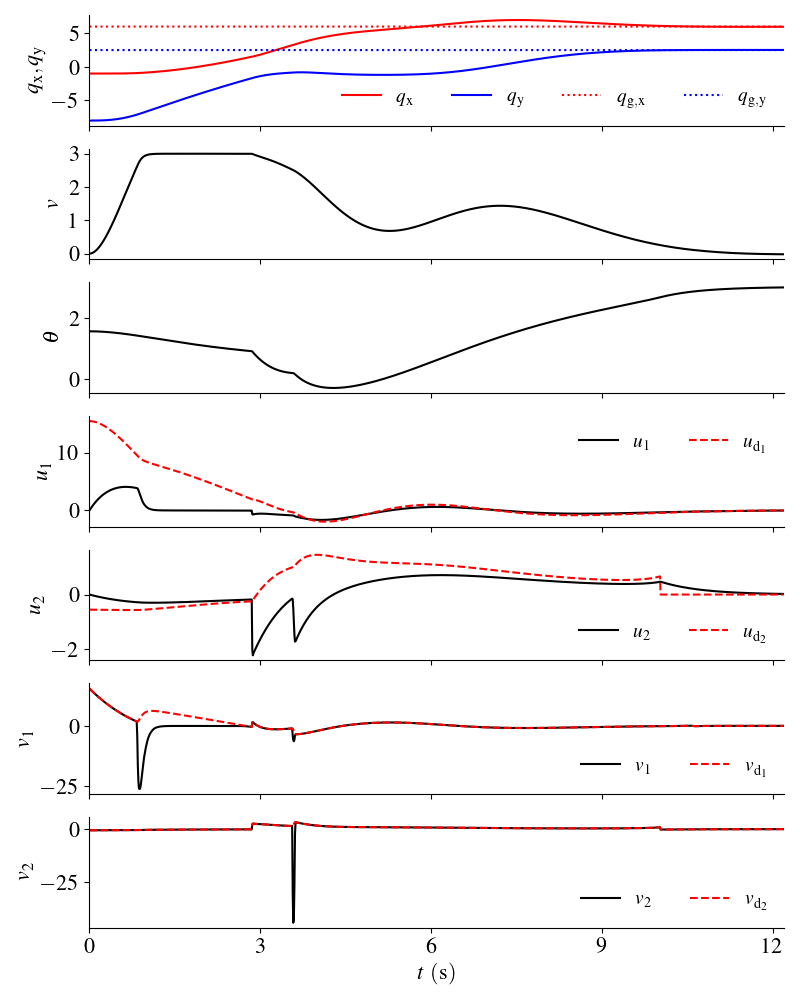}}
\caption{$q_\rmx$, $q_\rmy$, $v$, $\theta$, $u$, $u_\rmd$, $v =    [ \, v_1 \quad
    v_2 \,]^\rmT$, and $v_\rmd =    [ \, v_{\rmd_1} \quad
    v_{\rmd_2} \,]^\rmT$ for $q_{\rmg}=[\,6 \quad 2.5\,]^\rmT$.}
\label{fig:control_signal}
\end{figure} 

\begin{figure}[ht!]
\center{\includegraphics[width=0.46\textwidth,clip=true,trim= 0.15in 0.1in 0.1in 0.15in] {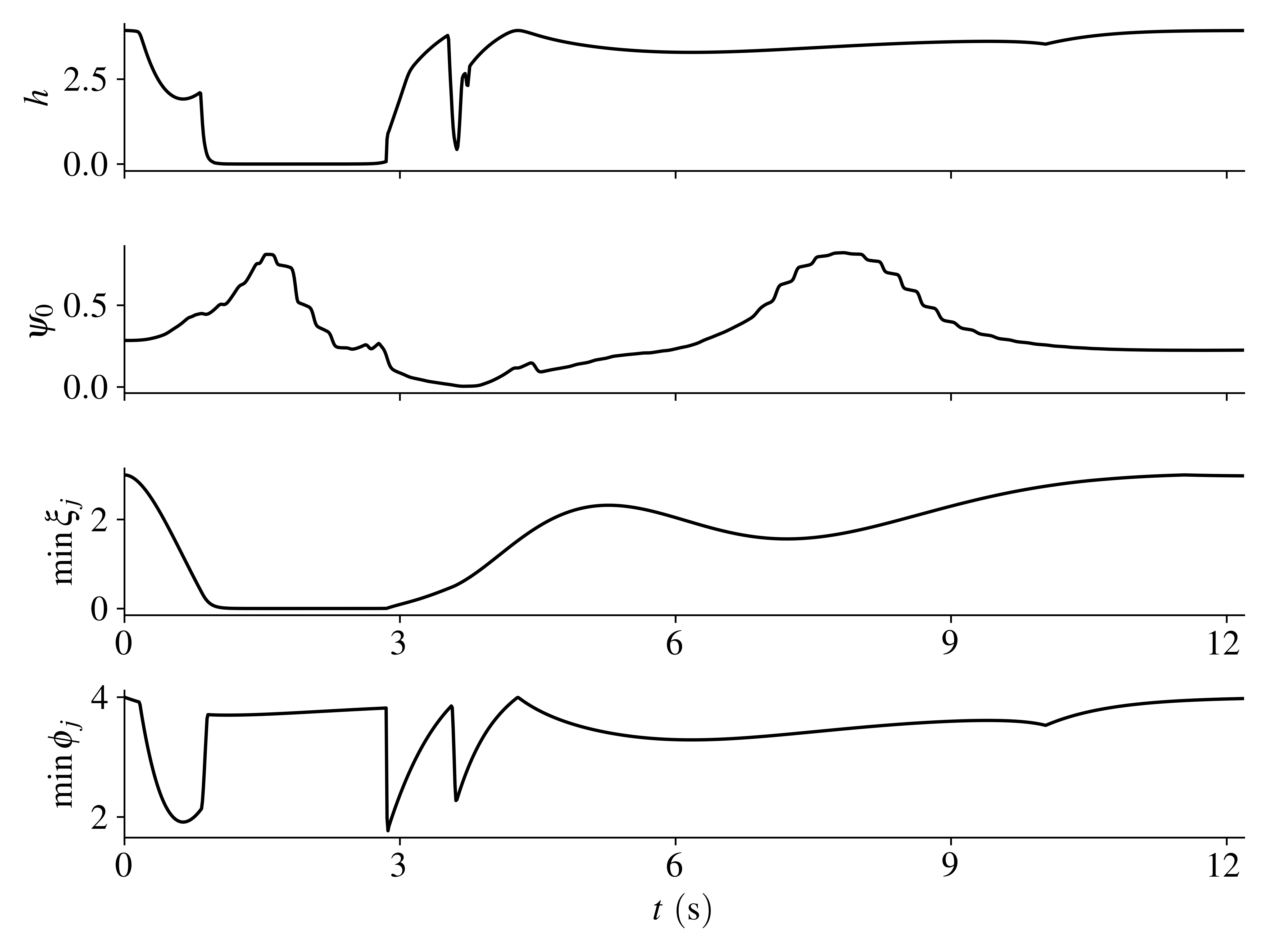}}
\caption{$h$, $\psi_0$, $\min \xi_j$, and $\min \phi_j$ for $q_\rmg = [\,6\quad2.5\,]^\rmT$~m.}
\label{fig:h_vals}
\end{figure}

\bibliographystyle{ieeetr}
\bibliography{LiDAR_StateComp}

\begin{thebibliography}{10}

\bibitem{hudson2021heterogeneous}
N.~Hudson {\em et~al.}, ``Heterogeneous ground and air platforms, homogeneous sensing: Team {CSIRO} {Data61}'s approach to the {DARPA} subterranean challenge,'' {\em arXiv preprint arXiv:2104.09053}, 2021.

\bibitem{kress2009temporal}
H.~Kress-Gazit, G.~E. Fainekos, and G.~J. Pappas, ``Temporal-logic-based reactive mission and motion planning,'' {\em IEEE Trans. Robotics}, vol.~25, no.~6, pp.~1370--1381, 2009.

\bibitem{schwarting2018planning}
W.~Schwarting, J.~Alonso-Mora, and D.~Rus, ``Planning and decision-making for autonomous vehicles,'' {\em Annual Review of Contr., Robotics, Autom. Sys.}, vol.~1, pp.~187--210, 2018.

\bibitem{wieland2007constructive}
P.~Wieland and F.~Allg{\"o}wer, ``Constructive safety using control barrier functions,'' {\em IFAC Proc.}, vol.~40, no.~12, pp.~462--467, 2007.

\bibitem{ames2019control}
A.~D. Ames, S.~Coogan, M.~Egerstedt, G.~Notomista, K.~Sreenath, and P.~Tabuada, ``Control barrier functions: Theory and applications,'' in {\em Proc. Europ. contr. conf.}, pp.~3420--3431, 2019.

\bibitem{ames2016control}
A.~D. Ames, X.~Xu, J.~W. Grizzle, and P.~Tabuada, ``Control barrier function based quadratic programs for safety critical systems,'' {\em IEEE Trans. Autom. Contr.}, pp.~3861--3876, 2016.

\bibitem{srinivasan2020synthesis}
M.~Srinivasan, A.~Dabholkar, S.~Coogan, and P.~A. Vela, ``Synthesis of control barrier functions using a supervised machine learning approach,'' in {\em Int. Conf. Int. Robots and Sys.}, pp.~7139--7145, IEEE, 2020.

\bibitem{khan2022gaussian}
M.~A. Khan, T.~Ibuki, and A.~Chatterjee, ``Gaussian control barrier functions: Non-parametric paradigm to safety,'' {\em IEEE Access}, vol.~10, pp.~99823--99836, 2022.

\bibitem{glotfelter2017nonsmooth}
P.~Glotfelter, J.~Cort{\'e}s, and M.~Egerstedt, ``Nonsmooth barrier functions with applications to multi-robot systems,'' {\em IEEE Contr. Sys. Lett.}, vol.~1, no.~2, pp.~310--315, 2017.

\bibitem{glotfelter2019hybrid}
P.~Glotfelter, I.~Buckley, and M.~Egerstedt, ``Hybrid nonsmooth barrier functions with applications to provably safe and composable collision avoidance for robotic systems,'' {\em IEEE Robotics and Autom. Lett.}, vol.~4, no.~2, pp.~1303--1310, 2019.

\bibitem{safari2024TSCT}
A.~Safari and J.~B. Hoagg, ``Time-varying soft-maximum barrier functions for safety in unmapped and dynamic environments,'' {\em arXiv preprint arXiv:2409.01458}, 2024.

\bibitem{wang2018}
L.~Wang, D.~Han, and M.~Egerstedt, ``Permissive barrier certificates for safe stabilization using sum-of-squares,'' in {\em 2018 Amer. Contr. Conf. (ACC)}, pp.~585--590, 2018.

\bibitem{tan2022compatibility}
X.~Tan and D.~V. Dimarogonas, ``Compatibility checking of multiple control barrier functions for input constrained systems,'' in {\em Proc. Conf. Dec. Contr. (CDC)}, pp.~939--944, IEEE, 2022.

\bibitem{backupautomatic}
P.~Rabiee and J.~B. Hoagg, ``Soft-minimum and soft-maximum barrier functions for safety with actuation constraints,'' {\em Automatica}, 2024 (to appear).

\bibitem{gurriet2020}
T.~Gurriet, M.~Mote, A.~Singletary, P.~Nilsson, E.~Feron, and A.~D. Ames, ``A scalable safety critical control framework for nonlinear systems,'' {\em IEEE Access}, pp.~187249--187275, 2020.

\bibitem{chen2020}
Y.~Chen, A.~Singletary, and A.~D. Ames, ``Guaranteed obstacle avoidance for multi-robot operations with limited actuation: A control barrier function approach,'' {\em IEEE Contr. Sys. Let.}, pp.~127--132, 2020.

\bibitem{compositionACC}
P.~Rabiee and J.~B. Hoagg, ``Composition of control barrier functions with differing relative degrees for safety under input constraints,'' in {\em Proc. Amer. Contr. Conf.}, 2024.

\bibitem{lindemann2018control}
L.~Lindemann and D.~V. Dimarogonas, ``Control barrier functions for signal temporal logic tasks,'' {\em IEEE Contr. Sys. Lett.}, vol.~3, no.~1, pp.~96--101, 2018.

\bibitem{safari2024ACC}
A.~Safari and J.~B. Hoagg, ``Time-varying soft-maximum control barrier functions for safety in an a priori unknown environment,'' in {\em 2024 American Control Conference (ACC)}, pp.~3698--3703, IEEE, 2024.

\bibitem{xiao2021high}
W.~Xiao and C.~Belta, ``High-order control barrier functions,'' {\em IEEE Trans. Autom. Contr.}, vol.~67, no.~7, pp.~3655--3662, 2021.

\bibitem{blanchini2008set}
F.~Blanchini {\em et~al.}, {\em Set-theoretic methods in control}.
\newblock Springer, 2008.

\bibitem{khalil2002control}
H.~K. Khalil, {\em Control of nonlinear systems}.
\newblock Prentice Hall, 2002.

\bibitem{de2002control}
A.~De~Luca, G.~Oriolo, and M.~Vendittelli, ``Control of wheeled mobile robots: An experimental overview,'' {\em RAMSETE: articulated and mobile robotics for services and technologies}, pp.~181--226, 2002.

\end{thebibliography}

\end{document}